\documentclass{article}
\usepackage[final, nonatbib]{nips_2018}
\usepackage[utf8]{inputenc}
\usepackage[T1]{fontenc}
\usepackage{hyperref} 
\usepackage{url}
\usepackage{booktabs}
\usepackage{amsfonts}
\usepackage{nicefrac}
\usepackage{microtype}
\usepackage{amsmath, amsthm}
\usepackage{graphicx}
\usepackage{enumitem}

\newtheorem{theorem}{Theorem}

\newtheorem{lemma}{Lemma}
\newtheorem{Lem}{Lemma}
\newtheorem{definition}{Definition}
\newtheorem{remark}{Remark}

\newcommand{\R}{{\mathbb R}}
\newcommand{\Z}{{\mathbb Z}}
\newcommand{\E}[1]{{\mathbb E}\left [#1\right]}
\renewcommand{\P}{{\mathbb P}}
\newcommand{\ep}{\varepsilon}

\newcommand{\gives}{\ensuremath{\rightarrow}}

\newcommand{\setst}[2]{\ensuremath{ \left\{ #1\,\right|\left.\,#2 \right\}}}

\newcommand{\abs}[1]{\ensuremath{\left| #1 \right|}}
\newcommand{\lr}[1]{\ensuremath{\left(#1 \right)}}
\newcommand{\norm}[1]{\left\lVert#1\right\rVert}
\newcommand{\inprod}[2]{\ensuremath{\left\langle#1,#2\right\rangle}}

\newcommand{\twiddle}[1]{\ensuremath{\widetilde{#1}}}

\newcommand{\dell}{\ensuremath{\partial}}
\newcommand{\set}[1]{\ensuremath{\{#1\}}}

\def\XXint#1#2#3{{\setbox0=\hbox{$#1{#2#3}{\int}$} \vcenter{\hbox{$#2#3$}}\kern-.5\wd0}}

\DeclareMathOperator{\Relu}{ReLU}

\DeclareMathOperator{\act}{act}
\DeclareMathOperator{\Act}{Act}
\DeclareMathOperator{\Var}{Var}

\title{Which Neural Net Architectures Give Rise to Exploding and
  Vanishing Gradients?}
\author{Boris Hanin}
\author{
  Boris Hanin\\
  Department of Mathematics\\
  Texas A\& M University\\
  College Station, TX, USA\\
  \texttt{bhanin@math.tamu.edu} 
}

\begin{document}
\maketitle 

\begin{abstract}
We give a rigorous analysis of the statistical behavior of gradients in a randomly initialized fully connected network $\mathcal N$ with $\Relu$ activations. Our results show that the empirical variance of the squares of the entries in the input-output Jacobian of $\mathcal N$ is exponential in a simple architecture-dependent constant $\beta,$ given by the sum of the reciprocals of the hidden layer widths. When $\beta$ is large, the gradients computed by $\mathcal N$ at initialization vary wildly. Our approach complements the mean field theory analysis of random networks. From this point of view, we rigorously compute finite width corrections to the statistics of gradients at the edge of chaos. 
\end{abstract}  

\section{Introduction}
A fundamental obstacle in training deep neural nets using gradient based optimization is the exploding and vanishing gradient problem (EVGP), which has attracted much attention (e.g.  \cite{bengio1994learning, hochreiter2001gradient,mishkin2015all,xie2017all,pennington2017resurrecting, pennington2018emergence}) after first being studied by Hochreiter \cite{hochreiter1991untersuchungen}. The EVGP occurs when the derivative of the loss in the SGD update 
\begin{equation}\label{E:param-update}
W\qquad\longleftarrow\qquad W~~-~~\lambda~ \frac{\dell\mathcal L}{\dell W},
\end{equation}
is very large for some trainable parameters $W$ and very small for others:
\[\abs{\frac{\dell\mathcal L}{\dell W}}~~\approx~~  0 ~~\text{or}~~ \infty.\]
This makes the increment in \eqref{E:param-update} either too small to be meaningful or too large to be precise. In practice, a number of ways of overcoming the EVGP have been proposed (see e.g. \cite{SchmidtHuber}). Let us mention three general approaches: (i) using architectures such as LSTMs \cite{hochreiter1997long}, highway networks \cite{srivastava2015highway}, or ResNets \cite{he2016deep} that are designed specifically to control gradients; (ii) precisely initializing weights (e.g. i.i.d. with properly chosen variances \cite{mishkin2015all,he2015delving} or using orthogonal weight matrices \cite{arjovsky2016unitary, henaff2016recurrent}); (iii) choosing non-linearities that that tend to compute numerically stable gradients or activations at initialization \cite{klambauer2017self}. 

A number of articles (e.g. \cite{poole2016exponential,schoenholz2016deep,pennington2017resurrecting, pennington2018emergence}) use mean field theory to show that even vanilla fully connected architectures can avoid the EVGP in the limit of \textit{infinitely wide} hidden layers. In this article, we continue this line of investigation. We focus specifically on fully connected $\Relu$ nets, and give a rigorous answer to the question of which combinations of depths $d$ and hidden layer widths $n_j$ give $\Relu$ nets that suffer from the EVGP at initialization. In particular, we avoid approach (iii) to the EVGP by setting once and for all the activations in $\mathcal N$ to be $\Relu$ and that we study approach (ii) in the limited sense that we consider only initializations in which weights and biases are independent (and properly scaled as in Definition \ref{D:rand-relu}) but do not investigate other initialization strategies. Instead, we focus on rigorously understanding the effects of finite depth and width on gradients in randomly initialized networks. The main contributions of this work are:\\
\begin{enumerate}
\item \textbf{We derive new exact formulas for the joint even moments of the entries of the input-output Jacobian in a fully connected $\Relu$ net with random weights and biases.} These formulas hold at finite depth and width (see Theorem \ref{T:path-moments}). \\

\item \textbf{We prove that the empirical variance of gradients in a fully connected $\Relu$ net is exponential in the sum of the reciprocals of the hidden layer widths.} This suggests that when this sum of reciprocals is too large, early training dynamics are very slow and it may take many epochs to achieve better-than-chance performance (see Figure \ref{fig:fixed_width}).\\

\item We prove that, so long as weights and biases are initialized independently with the correct variance scaling (see Definition  \ref{D:rand-relu}), \textbf{whether the EVGP occurs} (in the precise sense explained in \S \ref{S:EVGP}) \textbf{in fully connected $\Relu$ nets is a function only of the architecture and not the distributions from which the weights and biases are drawn}.\\

\end{enumerate}

\begin{figure}[htbp]
\begin{center}
\includegraphics[scale=0.5]{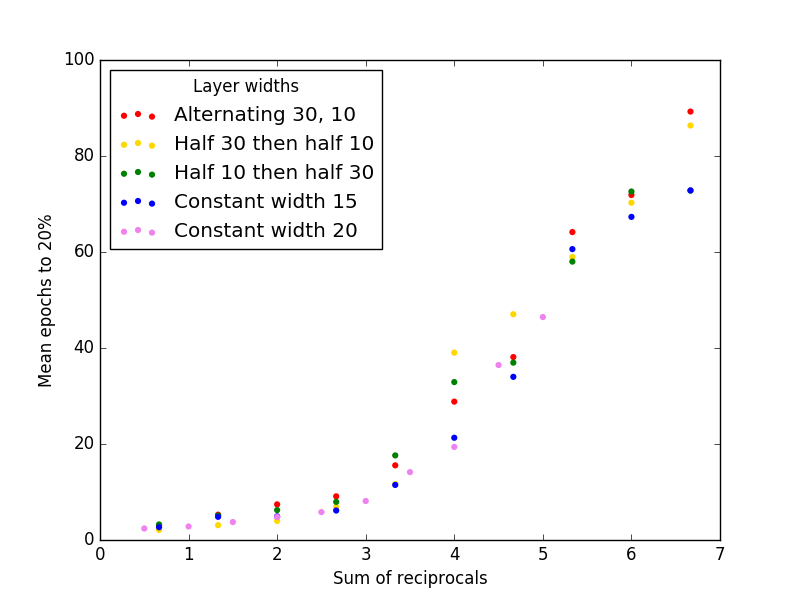}
\caption{Comparison of early training dynamics on vectorized MNIST for
  fully connected $\Relu$ nets with various architectures. Plot shows the mean number of epochs (over 100 independent training runs) that a given architecture takes to reach $20\%$ accuracy as a function of the sum of reciprocals of hidden layer widths. (Figure reprinted with permission from \cite{hanin2018start} with caption modified).}
\label{fig:fixed_width}
\end{center}
\end{figure}
\subsection{Practical Implications} The second of the listed contributions has several concrete consequences for architecture selection and for understanding initial training dynamics in $\Relu$ nets. Specifically, our main results, Theorems \ref{T:slope}-\ref{T:path-moments}, prove that the EVGP will occur in a $\Relu$ net $\mathcal N$ (in either the annealed or the quenched sense described in \S \ref{S:EVGP}) if and only if a single scalar parameter, the sum
\[\beta=\sum_{j=1}^{d-1}\frac{1}{n_j}\]
of reciprocals of the hidden layer widths of $\mathcal N,$ is large. Here $n_j$ denotes the width of the $j^{th}$ hidden layer, and we prove in Theorem \ref{T:slope} that the variance of entries in the input-output Jacobian of $\mathcal N$ is exponential in $\beta.$ Implications for architecture selection then follow from special cases of the power-mean inequality:
\begin{equation}\label{E:power-mean}
\lr{\frac{1}{d-1}\sum_{j=1}^{d-1} \frac{1}{n_j}}^{-1}~~\leq~~ \frac{1}{d-1}\sum_{j=1}^{d-1}n_j~~\leq~~ \lr{\frac{1}{d-1}\sum_{j=1}^{d-1}n_j^2}^{1/2},
\end{equation}
in which equality is achieved if and only if $n_j$ are all equal. We interpret the leftmost inequality as follows. Fix $d$ and a total budget $\sum_j n_j$ of hidden layer neurons. Theorems \ref{T:slope} and \ref{T:quenched} say that to avoid the EVGP in both the quenched and annealed senses, one should minimize $\beta$ and hence make the leftmost expression in \eqref{E:power-mean} as large as possible. This occurs precisely when $n_j$ are all equal. Fix instead $d$ and a budget of trainable parameters, $\sum_j n_{j}( n_{j-1}+1),$ which is close to $\sum_j n_j^2$ if the $n_j$'s don't fluctuate too much. Again using \eqref{E:power-mean}, we find that from the point of view of avoiding the EVGP, it is advantageous to take the $n_j$'s to be equal.

In short, our theoretical results (Theorems \ref{T:slope} and \ref{T:quenched}) show that if $\beta$ is large then, at initialization, $\mathcal N$ will compute gradients that fluctuate wildly, intuitively leading to slow initial training dynamics. This heuristic is corroborated by an experiment from \cite{hanin2018start} about the start of training on MNIST for fully connected neural nets with varying depths and hidden layer widths (the parameter $\beta$ appeared in \cite{hanin2018start} in a different context). Figure \ref{fig:fixed_width} shows that $\beta$ is a good summary statistic for predicting how quickly deep networks will start to train. 

We conclude the introduction by mentioning what we see as the principal weaknesses of the present work. First, our analysis holds only for $\Relu$ activations and assumes that all non-zero weights are independent and zero centered. Therefore, our conclusions do not directly carry over to convolutional, residual, and recurrent networks. Second, our results yield information about the fluctuations of the entries $Z_{p,q}$ of the input-output Jacobian $ J_{\mathcal N}$ at any fixed input to $\mathcal N.$ It would be interesting to have information about the joint distribution of the $Z_{p,q}$'s with inputs ranging over an entire dataset. Third, our techniques do not directly extend to initializations such as orthogonal weight matrices. We hope to address these issues in the future and, specifically, believe that the qualitative results of this article will generalize to convolutional networks in which the number of channels grows with the layer number.

\section{Relation to Prior Work} 
To provide some context for our results, we contrast both our approach and contributions with the recent work \cite{pennington2017resurrecting, pennington2018emergence}. These articles consider two senses in which a fully connected neural net $\mathcal N$ with random weights and biases can avoid the EVGP. The first is that the average singular value of the input-output Jacobian $J_{\mathcal N}$ remains approximately $1$, while the second, termed dynamical isometry, requires that all the singular values of $J_{\mathcal N}$ are approximately $1$. The authors of \cite{pennington2017resurrecting, pennington2018emergence} study the full distribution of the singular values of the Jacobian $J_{\mathcal N}$ first in the infinite width limit $n\gives \infty$ and then in the infinite depth limit $d\gives \infty.$ 

Let us emphasize two particularly attractive features of \cite{pennington2017resurrecting, pennington2018emergence}. First, neither the initialization nor the non-linearity in the neural nets $\mathcal N$ is assumed to be fixed, allowing the authors to consider solutions of types (ii) and (iii) above to the EVGP. The techniques used in these articles are also rather general, and point to the emergence of universality classes for singular values of the Jacobian of deep neural nets at initialization. Second, the results in these articles access the full distribution of singular values for the Jacobian $J_{\mathcal N}$, providing significantly more refined information than simply controlling the mean singular value. 

The neural nets considered in \cite{pennington2017resurrecting, pennington2018emergence} are essentially assumed to be infinitely wide, however. This raises the question of whether there is any finite width at which the behavior of a randomly initialized network will resemble the infinite width regime, and moreover, if such a width exists, how wide is wide enough? In this work we give rigorous answers to such questions by quantifying finite width effects, leaving aside questions about both different choices of non-linearity and about good initializations that go beyond independent weights. 

Instead of taking the singular value definition of the EVGP as in \cite{pennington2017resurrecting, pennington2018emergence}, we propose two non-spectral formulations of the EVGP, which we term annealed and quenched. Their precise definitions are given in \S\ref{S:annealed} and \S \ref{S:quenched}, and we provide in \S \ref{S:EVGP-comp} a discussion of the relation between the different senses in which the EVGP can occur.

Theorem \ref{T:slope} below implies, in the infinite width limit, that all $\Relu$ nets avoid the EVGP in both the quenched and annealed sense. Hence, our definition of the EVGP (see \S \ref{S:annealed} and \S \ref{S:quenched}) is weaker than the dynamical isometry condition from \cite{pennington2017resurrecting,pennington2018emergence}. But, as explained in \S \ref{S:EVGP-comp}, it is stronger the condition that the average singular value equal $1.$ Both the quenched and annealed versions of the EVGP concern the fluctuations of the partial derivatives 
\begin{equation}\label{E:zpq-def}
Z_{p,q}:=\frac{\dell \lr{f_{\mathcal N}}_q}{\dell \Act_p^{(0)}}
\end{equation}
of the $q^{th}$ component of the function $f_{\mathcal N}$ computed by $\mathcal N$ with respect to the $p^{th}$ component of its input ($\Act^{(0)}$ is an input vector - see \eqref{E:act-def}). The stronger, quenched version of the EVGP concerns the \textit{empirical variance} of the squares of all the different $Z_{p,q}:$
\begin{equation}\label{E:empvar-def}
\widehat{\Var}\left[Z^2\right]:=\frac{1}{M}\sum_{m=1}^M Z_{p_m,q_m}^4 - \lr{\frac{1}{M}\sum_{m=1}^M Z_{p_m,q_m}^2}^2,\qquad M= n_0n_d.
\end{equation}
Here, $n_0$ is the input dimension to $\mathcal N$, $n_d$ is the output dimension, and the index $m$ runs over all $n_0n_d$ possible input-output neuron pairs $(p_m,q_m).$ Intuitively, since we will show in Theorem \ref{T:slope} that
\[\E{Z_{p,q}^2} = \Theta(1),\]
independently of the depth, having a large mean for $\widehat{\Var}\left[Z^2\right]$ means that for a typical realization of the weights and biases in $\mathcal N$, the derivatives of $f_{\mathcal N}$ with respect to different trainable parameters will vary over several orders of magnitude, leading to inefficient SGD updates \eqref{E:param-update} for any fixed learning rate $\lambda$ (see \S \ref{S:EVGP-comp} - \S \ref{S:quenched}).

To avoid the EVGP (in the annealed or quenched senses described below) in deep feed-forward networks with $\Relu$ activations, our results advise letting the widths of hidden layers grow as a function of the depth. In fact, as the width of a given hidden layer tends to infinity, the input to the next hidden layer can viewed as a Gaussian process and can be understood using mean field theory (in which case one first considers the infinite width limit and only then the infinite depth limit). This point of view was taken in several interesting papers (e.g. \cite{poole2016exponential,schoenholz2016deep, pennington2017resurrecting, pennington2018emergence} and references therein), which analyze the dynamics of signal propagation through such deep nets. In their notation, the fan-in normalization (condition (ii) in Definition \ref{D:rand-relu}) guarantees that we've initialized our neural nets at the edge of chaos (see e.g. around (7) in \cite{poole2016exponential} and (5) in \cite{schoenholz2016deep}). Indeed, writing $\mu^{(j)}$ for the weight distribution at layer $j$ and using our normalization $\mathrm{Var}[\mu^{(j)}]=2/n_{j-1}$, the order parameter $\chi_1$ from \cite{poole2016exponential, schoenholz2016deep} becomes 
\begin{align*}
\chi_1&=n_{j-1}\cdot \mathrm{Var}[\mu^{(j)}] \int_{\R} e^{-z^2/2} \lr{\phi'(\sqrt{q^*}z)}^2\frac{dz}{\sqrt{2\pi}}=1,
\end{align*}
since $\phi=\Relu$, making $\phi'(z)$ the indicator function ${\bf 1}_{[0,\infty)}(z)$ and the value of $\phi'(\sqrt{q^*}z)$ independent of the asymptotic length $q^*$ for activations. The condition $\chi_1=1$ defines the edge of chaos regime. This gives a heuristic explanation for why the nets considered in the present article cannot have just one of vanishing and exploding gradients. It also allows us to interpret our results as a rigorous computation for $\Relu$ nets of the $1/n_j$ corrections at the edge of chaos.

In addition to the mean field theory papers, we mention the article \cite{schoenholz2017correspondence}. It does not deal directly with gradients, but it does treat the finite width corrections to the statistical distribution of pre-activations in a feed-forward network with Gaussian initialized weights and biases. A nice aspect of this work is that the results give the joint distribution not only over all the neurons but also over any number of inputs to the network.  In a similar vein, we bring to the reader's attention \cite{balduzzi2017shattered}, which gives interesting heuristic computations about the structure of correlations between gradients corresponding to different inputs in both fully connected and residual $\Relu$ nets.


\section{Defining the EVGP for Feed-Forward Networks}\label{S:EVGP}
We now explain in exactly what sense we study the EVGP and contrast our definition, which depends on the behavior of the \textit{entries} of the input-output Jacobian $J_{\mathcal N}$, with the more usual definition, which depends on the behavior of its singular values (see \S \ref{S:EVGP-comp}). To do this, consider a feed-forward fully connected depth $d$ network $\mathcal N$ with hidden layer widths $n_0,\ldots, n_d,$ and fix an input $\Act^{_{(0)}}\in \R^{n_0}$. We denote by $\Act^{_{(j)}}$ the corresponding vector of activations at layer $j$ (see \eqref{E:act-def}). The exploding and vanishing gradient problem can be roughly stated as follows:
\begin{equation}\label{E:EVGP-def}
\mathrm{Exploding/Vanishing~ Gradients}\qquad \longleftrightarrow \qquad Z_{p,q} \text{ has large fluctuations},
\end{equation}
where $Z_{p,q}$ the entries of the Jacobian $J_{\mathcal N}$ (see \eqref{E:zpq-def}). A common way to formalize this statement is to interpret ``$Z_{p,q}$ has large fluctuations'' to mean that the Jacobian $J_{\mathcal N}$ of the function computed by $\mathcal N$ has both very large and very small singular values \cite{bengio1994learning, hochreiter2001gradient,pennington2017resurrecting}. We give in \S \ref{S:EVGP-comp} a brief account of the reasoning behind this formulation of the EVGP and explain why is also natural to define the EVGP via the moments of $Z_{p,q}.$ Then, in \S \ref{S:annealed} and \S \ref{S:quenched}, we define two precise senses, which we call annealed and quenched, in which that EVGP can occur, phrased directly in terms of the joint moments of $Z_{p,q}.$ 


\subsection{Spectral vs. Entrywise Definitions of the EVGP}\label{S:EVGP-comp} Let us recall the rationale behind using the spectral theory of $J_{\mathcal N}$ to define the EVGP. The gradient in \eqref{E:param-update} of the loss with respect to, say, a weight $W_{\alpha,\beta}^{_{(j)}}$ connecting neuron $\alpha$ in layer $j-1$ to neuron $\beta$ in layer $j$ is
\begin{equation}
\partial \mathcal L/\partial W_{\alpha,\beta}^{(j)}=\inprod{\nabla_{\Act^{(d)}}\mathcal L}{J_{\mathcal N,\beta}(j\gives d)}\Act_\alpha^{(j-1)}\phi'(\Act_\beta^{(j)}),\label{E:SGD-expand}
\end{equation}
where $\phi'(\Act_\beta^{(j)})$ is the derivative of the non-linearity, the derivative of the loss $\mathcal L$ with respect to the output $\Act^{_{(d)}}$ of $\mathcal N$ is
\[\nabla_{\Act^{(d)}}\mathcal L=\lr{\partial \mathcal L\big/\partial \Act_{q}^{(d)},\,q=1,\ldots, n_d},\]
and we've denoted the $\beta^{th}$ row in the layer $j$ to output Jacobian $J_{\mathcal N}(j\gives d)$ by
\[ J_{\mathcal N,\beta}(j\gives d)=\lr{\partial \Act_q^{(d)}\big/\partial \Act_\beta^{(j)},\,q=1,\ldots, n_d }.\]

Since $J_{\mathcal N}(j\gives d)$ is the product of $d-j$ layer-to-layer Jacobians, its inner product with $\nabla_{\Act^{(d)}}\mathcal L$ is usually the term considered responsible for the EVGP. The \textit{worst case} distortion it can achieve on the vector $\nabla_{\Act^{(d)}}\mathcal L$ is captured precisely by its condition number, the ratio of its largest and smallest singular values. 

However, unlike the case of recurrent networks in which $J_{\mathcal N}(j\gives d)$ is $(d-j)-$fold product of a fixed matrix, when the hidden layer widths grow with the depth $d$, the dimensions of the layer $j$ to layer $j'$ Jacobians $J_{\mathcal N}(j\gives j')$ are not fixed and it is not clear to what extent the vector $\nabla_{\Act^{(d)}}\mathcal L$ will actually be stretched or compressed by the worst case bounds coming from estimates on the condition number of $J_{\mathcal N}(j\gives d).$

Moreover, on a practical level, the EVGP is about the numerical stability of the increments of the SGD updates \eqref{E:param-update} over all weights (and biases) in the network, which is directly captured by the joint distribution of the random variables
\[\set{|\partial \mathcal L/ \partial W_{\alpha,\beta}^{(j)}|^2,\,\, j=1,\ldots, n_d,\,\,\alpha=1,\ldots, n_{j-1},\,\beta=1,\ldots,n_{j}}.\]
Due to the relation \eqref{E:SGD-expand}, two terms influence the moments of $|\partial \mathcal L/ \partial W_{\alpha,\beta}^{(j)}|^2$: one coming from the \textit{activations} at layer $j-1$ and the other from the \textit{entries} of $J_{\mathcal N}(j\gives d).$ We focus in this article on the second term and hence interpret the fluctuations of the entries of $J_{\mathcal N}(j\gives d)$ as a measure of the EVGP.

To conclude, we recall a simple relationship between the moments of the entries of the input-output Jacobian $J_{\mathcal N}$ and the distribution of its singular values, which can be used to directly compare spectral and entrywise definitions of the EVGP. Suppose for instance one is interested in the average singular value of $J_{\mathcal N}$ (as in \cite{pennington2017resurrecting, pennington2018emergence}). The sum of the singular values of $J_{\mathcal N}$ is given by 
\[\mathrm{tr}(J_{\mathcal N}^TJ_{\mathcal N}) = \sum_{j=1}^{n_0} \inprod{J_{\mathcal N}^TJ_{\mathcal N} u_j}{u_j} = \sum_{j=1}^{n_0} \norm{J_{\mathcal N}u_j}^2,\]
where $\set{u_j}$ is any orthonormal basis. Hence, the average singular value can be obtained directly from the joint even moments of the entries of $J_{\mathcal N}$. Both the quenched and annealed EVGP (see \eqref{E:EVGP-annealed},\eqref{E:EVGP-quenched}) entail that the average singular value for $J_{\mathcal N}$ equals $1,$ and we prove in Theorem \ref{T:slope} (specifically \eqref{E:Z2}) that even at finite depth and width the average singular value for $J_{\mathcal N}$ equals $1$ for all the random $\Relu$ nets we consider!

One can push this line of reasoning further. Namely, the singular values of any matrix $M$ are determined by the Stieltjes transform of the empirical distribution $\sigma_{M}$ of the eigenvalues of $M^TM:$
\[S_{M}(z)=\int_{\R}\frac{d\sigma_{M}(x)}{z-x},\qquad z \in \mathbb{C}\backslash \R.\]
Writing $(z-x)^{-1}$ as a power series in $z$ shows that $S_{J_{\mathcal N}}$ is determined by traces of powers of $J_{\mathcal N}^TJ_{\mathcal N}$ and hence by the joint even moments of the entries of $J_{\mathcal N}$. We hope to estimate $S_{J_{\mathcal N}}(z)$ directly in future work.

\subsection{Annealed Exploding and Vanishing Gradients}\label{S:annealed}
Fix a sequence of positive integers $n_0,n_1,\ldots .$ For each $d\geq 1$ write $\mathcal N_d$ for the depth $d$ $\Relu$ net with hidden layer widths $n_{0},\ldots, n_{d}$ and random weights and biases (see Definition \ref{D:rand-relu} below). As in \eqref{E:zpq-def}, write $Z_{p,q}(d)$ for the partial derivative of the $q^{th}$ component of the output of $\mathcal N_d$ with respect to $p^{th}$ component of its input. We say that the family of architectures given by $\set{n_0,n_1,\ldots }$ \textit{avoids the exploding and vanishing gradient problem in the annealed sense} if for each fixed input to $\mathcal N_d$ and every $p,q$ we have
\begin{equation}\label{E:EVGP-annealed}
\E{Z_{p,q}^2(d)}=1,\quad \Var[Z_{p,q}^2(d)]=\Theta(1),\quad \sup_{d\geq 1}~ \E{Z_{p,q}^{2K}(d)}~~<~~\infty,~\,\,\forall K\geq 3.
\end{equation}
Here the expectation is over the weights and biases in $\mathcal N_d$. Architectures that avoid the EVGP in the annealed sense are ones where the typical magnitude of the partial derivatives $Z_{p,q}(d)$ have bounded (both above and below) fluctuations around a constant mean value. This allows for a reliable a priori selection of the learning rate $\lambda$ from \eqref{E:param-update} even for deep architectures. Our main result about the annealed EVGP is Theorem \ref{T:slope}: a family of neural net architectures avoids the EVGP in the annealed sense if and only if 
\begin{equation}\label{E:finite-var}
\sum_{j=1}^{\infty}\frac{1}{n_j}~~<~~\infty.
\end{equation}
We prove in Theorem \ref{T:slope} that $\E{Z_{p,q}^{2K}(d)}$ is exponential in $\sum_{j\leq d}1/n_j$ for every $K.$

\subsection{Quenched Exploding and Vanishing Gradients}\label{S:quenched}
There is an important objection to defining the EVGP as in the previous section. Namely, if a neural net $\mathcal N$ suffers from the annealed EVGP,  then it is impossible to choose an appropriate a priori learning rate $\lambda$ that works for a typical initialization. However, it may still be that for a typical realization of the weights and biases there is some choice of $\lambda$ (depending on the particular initialization), that works well for all (or most) trainable parameters in $\mathcal N.$ To study whether this is the case, we must consider the variation of the $Z_{p,q}$'s across different $p,q$ in a fixed realization of weights and biases. This is the essence of the quenched EVGP.

To formulate the precise definition, we again fix a sequence of positive integers $n_0,n_1,\ldots$ and write $\mathcal N_d$ for a depth $d$ $\Relu$ net with hidden layer widths $n_0,\ldots, n_d.$ We write as in \eqref{E:empvar-def}
\begin{equation*}
\widehat{\Var}\left[Z(d)^2\right]:=\frac{1}{M}\sum_{m=1}^M Z_{p_m,q_m}(d)^4 - \lr{\frac{1}{M}\sum_{m=1}^M Z_{p_m,q_m}(d)^2}^2,\qquad M = n_0n_d
\end{equation*}
for the empirical variance of the squares all the entries $Z_{p,q}(d)$ of the input-output Jacobian of $\mathcal N_d.$ We will say that the family of architectures given by $\set{n_{0},n_1,\ldots}$ \textit{avoids the exploding and vanishing gradient problem in the quenched sense} if
\begin{equation}\label{E:EVGP-quenched}
\E{Z_{p,q}(d)^2}=1\qquad \text{and}\qquad \E{\widehat{\Var}[Z(d)^2]}=\Theta(1).
\end{equation}
Just as in the annealed case \eqref{E:EVGP-annealed}, the expectation $\E{\cdot}$ is with respect to the weights and biases of $\mathcal N.$ In words, a neural net architecture suffers from the EVGP in the quenched sense if for a typical realization of the weights and biases the empirical variance of the squared partial derivatives $\set{Z_{p_m,q_m}^2}$ is large. 

Our main result about the quenched sense of the EVGP is Theorem \ref{T:quenched}. It turns out, at least for the $\Relu$ nets we study, that a family of neural net architectures avoids the quenched EVGP if and only if it also avoids the annealed exploding and vanishing gradient problem (i.e. if \eqref{E:finite-var} holds). 

\section{Acknowledgements} 
I thank Leonid Hanin for a number of useful conversations and for his comments on an early draft. I am also grateful to Jeffrey Pennington for pointing out an important typo in the proof of Theorem \ref{T:path-moments} and to David Rolnick for several helpful conversations and, specifically, for pointing out the relevance of the power-mean inequality for understanding $\beta$. Finally, I would like to thank several anonymous referees for their help in improving the exposition. One referee in particular raised concerns about the annealed and quenched definitions of the EVGP. Addressing these concerns resulted in the discussion in \S \ref{S:EVGP-comp}.

\section{Notation and Main Results}
\subsection{Definition of Random Networks} \label{S:model-def}
To formally state our results, we first give the precise definition of the random networks we study. For every $d\geq 1$ and each ${\bf n}=\lr{n_i}_{i=0}^d\in \Z_+^{d+1}$, write
\[\mathfrak N({\bf n},d)=\left\{\substack{ \text{fully connected feed-forward nets with }\Relu\text{ activations,}\\\text{  depth }d,  \text{ and whose } j^{th} \text{ hidden layer has width }n_j}\right\}.\]
The function $f_{\mathcal N}$ computed by $\mathcal N \in \mathfrak N({\bf n},d)$ is determined by a collection of weights and biases
\[\set{w_{\alpha,\beta}^{(j)},\, b_\beta^{(j)},\quad 1\leq \alpha \leq n_j,~~1\leq \beta \leq n_{j+1},\,\, j=0,\ldots, d-1}.\]
Specifically, given an input
\[\Act^{(0)}=\lr{\Act_i^{(0)}}_{i=1}^{n_0}\in \R^{n_0}\]
to $\mathcal N,$ we define for every $j=1,\ldots, d$
\begin{equation}\label{E:act-def}
\act_\beta^{(j)}= b_{\beta}^{(j)}+\sum_{\alpha = 1}^{n_{j-1}} \Act_\alpha^{(j-1)}w_{\alpha, \beta}^{(j)},\qquad \Act_\beta^{(j)}=\phi(\act_\beta^{(j)}),\quad 1\leq \beta \leq n_j.
\end{equation}
The vectors $\act^{(j)},\, \Act^{(j)}$ therefore represent the vectors of inputs and outputs of the neurons in the $j^{th}$ layer of $\mathcal N.$ The function computed by $\mathcal N$ takes the form
\[f_{\mathcal N}\lr{\Act^{(0)}}=f_{\mathcal N}\lr{\Act^{(0)}, w_{\alpha, \beta}^{(j)}, b_{\beta}^{(j)}}=\Act^{(d)}.\]
A random network is obtained by randomizing weights and biases. 

\begin{definition}[Random Nets]\label{D:rand-relu}
 Fix $d\geq 1, \, {\bf n}=\lr{n_0,\ldots, n_d}\in \Z_{+}^{d+1},$ and two collections of probability measures ${\bf \mu}=\lr{\mu^{(1)},\ldots, \mu^{(d)}}$ and ${\bf \nu}=\lr{\nu^{(1)},\ldots, \nu^{(d)}}$ on $\R$ such that
  \begin{itemize}
  \item[(i)] $\mu^{(j)},\nu^{(j)}$ are symmetric around $0$ for every $1\leq j \leq d.$
  \item[(ii)] the variance of $\mu^{(j)}$ is $2/(n_{j-1})$.
  \item[(iii)] $\nu^{(j)}$ has no atoms.
  \end{itemize}
A random net $\mathcal N \in \mathfrak N_{{\bf \mu},{\bf \nu}}\lr{{\bf n},d}$ is obtained by requiring that the weights and biases for neurons at layer $j$ are drawn independently from $\mu^{(j)},\nu^{(j)}:$
\[w_{\alpha,\beta}^{(j)}~~\sim~~\mu^{(j)}, \,\, b_{\beta}^{(j)}~~\sim~~\nu^{(j)}\qquad i.i.d.\]
\end{definition}
\begin{remark}\label{R:atoms}
 Condition (iii) is used when we apply Lemma \ref{L:key} in the proof of Theorem \ref{T:path-moments}. It can be removed under the restriction that $d\ll \exp\lr{\sum_{j=1}^d n_j}.$ Since this yields slightly messier but not meaningfully different results, we do not pursue this point.   
\end{remark}

\subsection{Results}\label{S:thms} Our main theoretical results are Theorems \ref{T:slope} and \ref{T:path-moments}. They concern the statistics of the slopes of the functions computed by a random neural net in the sense of Definition \ref{D:rand-relu}. To state them compactly, we define for any probability measure $\mu$ on $\R$
\[\twiddle{\mu}_{2K}:=\frac{\int_\R x^{2K} d\mu}{\lr{\int_Rx^2 d\mu}^{K}},\qquad K\geq 0,\]
and, given a collection of probability measures $\set{\mu^{(j)}}_{j=1}^d$ on $\R$, set for any $K\geq 1$
\[\twiddle{\mu}_{2K,max}:=\max_{1\leq j\leq d} \twiddle{\mu}_{2K}^{(j)}.\]
We also continue to write $Z_{p,q}$ for the entries of the input-output Jacobian of a neural net (see \eqref{E:zpq-def}). 

\begin{theorem}\label{T:slope}
   Fix $d\geq 1$ and a multi-index ${\bf n}=\lr{n_0,\ldots, n_d}\in \Z_+^{d+1}.$ Let $\mathcal N\in \mathfrak N_{{\bf \mu},{\bf \nu}}\lr{{\bf n},d}$ be a random network as in Definition \ref{D:rand-relu}. For any fixed input to $\mathcal N$, we have
   \begin{equation}\label{E:Z2}
\E{Z_{p,q}^2}=\frac{1}{n_0}.
\end{equation}
In contrast, the fourth moment of $Z_{p,q}(x)$ is exponential in $\sum_j \frac{1}{n_j}:$
\begin{equation}\label{E:Z4}
\frac{2}{n_0^2}\exp\lr{\frac{1}{2}\sum_{j=1}^{d-1}\frac{1}{n_j}}\leq\E{Z_{p,q}^4}\leq  \frac{6\twiddle{\mu}_{4,max}}{n_0^2}\exp\lr{6~\twiddle{\mu}_{4,max} \sum_{j=1}^{d-1} \frac{1}{n_j}}.
\end{equation}
Moreover, there exists a constant $C_{K,\mu}>0$ depending only on $K$ and the first $2K$ moments of the measures $\set{\mu^{(j)}}_{j=1}^d$ such that if $K< \min_{j=1}^{d-1} \set{n_j}$, then
\begin{equation}\label{E:Z2K}
  \E{Z_{p,q}^{2K}}\leq \frac{C_{K,\mu}}{n_0^K}\exp\lr{C_{K,\mu} \sum_{j=1}^{d-1}\frac{1}{n_j}}.
\end{equation}
\end{theorem} 
\begin{remark}
In \eqref{E:Z2}, \eqref{E:Z4}, and \eqref{E:Z2K}, the bias distributions $\nu^{(j)}$ play no role. However, in the derivation of these relations, we use in Lemma \ref{L:key} that $\nu^{(j)}$ has no atoms (see Remark \ref{R:atoms}). Also, the condition $K < \min_{j=1}^d \set{n_{j-1}}$ can be relaxed by allowing $K$ to violate this inequality a fixed finite number $\ell$ of times. This causes the constant $C_{K,\mu}$ to depend on $\ell$ as well. 
\end{remark}
We prove Theorem \ref{T:slope} in Appendix \ref{S:slope-pf}. The constant factor multiplying $\sum_j 1/n_j$ in the exponent on the right hand side of \eqref{E:Z4} is not optimal and can be reduced by a more careful analysis along the same lines as the proof of Theorem \ref{T:slope} given below. We do not pursue this here, however, since we are primarily interested in fixing $K$ and understanding the dependence of $\E{Z_{p,q}(x)^{2K}}$ on the widths $n_j$ and the depth $d.$ Although we've stated Theorem \ref{T:slope} only for the even moments of $Z_{p,q}$, the same techniques will give analogous estimates for any mixed even moments $\E{Z_{p_1,q}^{2K_1}\cdots Z_{p_m, q}^{2K_m}}$ when $K$ is set to $\sum_m K_m$ (see Remark \ref{R:mixed-moments}). In particular, we can estimate the mean of the empirical variance of gradients. 
\begin{theorem}\label{T:quenched}
  Fix $n_0,\ldots,n_d \in \Z_+,$ and let $\mathcal N$ be a random fully connected depth $d$ $\Relu$ net with hidden layer widths $n_0,\ldots, n_d$ and random weights and biases as in Definition \ref{D:rand-relu}. Write $M=n_0n_d$ and write $\widehat{\Var}\left[Z^2\right]$ for the empirical variance of the squares $\set{Z_{p_m,q_m}^2}$ of all $M$ input-output neuron pairs as in \eqref{E:empvar-def}. We have
\begin{equation}\label{E:quenched-est-ub}
\E{\widehat{\Var}[Z^2]}~~\leq~~ \lr{1-\frac{1}{M}}  \frac{6\twiddle{\mu}_{4,max}}{n_0^2} \exp \lr{6~\twiddle{\mu}_{4,max}\sum_{j=1}^{d-1}\frac{1}{n_j}}
\end{equation}
and
\begin{equation}\label{E:quenched-est-lb}
\E{\widehat{\Var}[Z^2]}~~\geq ~~ \frac{1}{n_0^2}\lr{1-\frac{1}{M}}\lr{ 1-\eta  + \frac{4\eta}{n_1}\lr{\twiddle{\mu}_4^{(1)}-1}e^{-\frac{1}{n_1}}  } \exp\lr{\frac{1}{2}\sum_{j=1}^{d-1}\frac{1}{n_j}}
\end{equation}
where
\[\eta:=\frac{\#\setst{m_1,m_2}{m_1\neq m_2,\,\, q_{m_1}=q_{m_2}}}{M(M-1)}=\frac{n_0-1}{n_0n_d-1}.\]
Hence, the family $\mathcal N_d$ of $\Relu$ nets avoids the exploding and vanishing gradient problem in the quenched sense if and only if
\[\sum_{j=1}^\infty \frac{1}{n_j}<\infty.\]
\end{theorem}

We prove Theorem \ref{T:quenched} in Appendix \ref{S:quenched-pf}. The results in Theorems \ref{T:slope} and \ref{T:quenched} are based on \textit{exact} expressions, given in Theorem \ref{T:path-moments}, for the even moments $\E{Z_{p,q}(x)^{2K}}$ in terms only of the moments of the weight distributions $\mu^{(j)}$. To give the formal statement, we introduce the following notation. For any ${\bf n}=\lr{n_i}_{i=0}^d$ and any $1\leq p \leq n_0,\, 1\leq q \leq n_d,$ we say that a path $\gamma$ from the $p^{th}$ input neuron to the $q^{th}$ output neuron in $\mathcal N\in \mathfrak N\lr{{\bf n},d}$ is a sequence 
\[\set{\gamma(j)}_{j=0}^d,\qquad 1\leq \gamma(j)\leq n_j,\quad \gamma(0)=p,\quad \gamma(d)=q,\]
so that $\gamma(j)$ represents a neuron in the $j^{th} $ layer of $\mathcal N.$ Similarly, given any collection of $K\geq 1$ paths $\Gamma=\lr{\gamma_k}_{k=1}^K$ that connect (possibly different) neurons in the input of $\mathcal N$ with neurons in its output and any $1\leq j \leq d$, denote by
\[\Gamma(j)=\bigcup_{\gamma \in \Gamma}\set{\gamma(j)}\]
the neurons in the $j^{th}$ layer of $\mathcal N$ that belong to at least one element of $\Gamma.$ Finally, for every $\alpha \in \Gamma(j-1)$ and $\beta\in \Gamma(j)$, denote by 
\[\abs{\Gamma_{\alpha,\beta}(j)}=\#\setst{\gamma \in \Gamma}{\gamma(j-1)=\alpha,\,\gamma(j)=\beta}\]
the number of paths in $\Gamma$ that pass through neuron $\alpha$ at layer $j-1$ and through neuron $\beta$ at layer $j.$
\begin{theorem}\label{T:path-moments}
 Fix $d\geq 1$ and ${\bf n}=\lr{n_0,\ldots, n_d}\in \Z_+^{d+1}.$ Let $\mathcal N\in \mathfrak N_{{\bf \mu},{\bf \nu}}\lr{{\bf n},d}$ be a random network as in Definition \ref{D:rand-relu}. For every $K\geq 1$ and all $1\leq p \leq n_0,\,\, 1\leq q \leq n_d$, we have
  \begin{equation}\label{E:path-moments}
    \E{Z_{p,q}^{2K}}=\sum_{\Gamma} \prod_{j=1}^d C_j(\Gamma),
  \end{equation}
where the sum is over ordered tuples $\Gamma=\lr{\gamma_1,\ldots, \gamma_{2K}}$ of paths in $\mathcal N$ from $p$ to $q$ and
\[C_j(\Gamma) = \lr{\frac{1}{2}}^{\abs{\Gamma(j)}}\prod_{\substack{\alpha\in \Gamma(j-1)\\ \beta\in \Gamma(j)}}   \mu_{\abs{\Gamma_{\alpha,\beta}(j)}}^{(j)},\]
where for every $r\geq 0,$ the quantity $\mu_r^{(j)}$ denotes the $r^{th}$ moment of the measure $\mu^{(j)}.$
\end{theorem}
\begin{remark}\label{R:mixed-moments}
The expression \eqref{E:path-moments} can be generalized to case of mixed even moments. Namely, given $m\geq 1$ and for each $1\leq m \leq M$ integers $K_m\geq 0$ and $1\leq p_m\leq n_0\,\, 1\leq q_m\leq n_d$, we have
\begin{equation}\label{E:mixed-moments}
    \E{\prod_{m=1}^M Z_{p_m,q_m}(x)^{2K_m}}=\sum_{\Gamma} \prod_{j=1}^d C_j(\Gamma),
  \end{equation}
where now the sum is over collections $\Gamma=\lr{\gamma_1,\ldots, \gamma_{2K}}$ of $2K=\sum_m 2K_m$ paths in $\mathcal N$ with exactly $2K_m$ paths from $p_m$ to $q_m.$ The proof is identical up to the addition of several well-placed subscripts. 
\end{remark}
See Appendix \ref{S:path-moments-pf} for the proof of Theorem \ref{T:path-moments}. 


\bibliographystyle{alpha}
  \bibliography{DL_bibliography}

\newcommand{\etalchar}[1]{$^{#1}$}
\begin{thebibliography}{KUMH17}

\bibitem[ASB16]{arjovsky2016unitary}
Martin Arjovsky, Amar Shah, and Yoshua Bengio.
\newblock Unitary evolution recurrent neural networks.
\newblock In {\em International Conference on Machine Learning}, pages
  1120--1128, 2016.

\bibitem[BFL{\etalchar{+}}17]{balduzzi2017shattered}
David Balduzzi, Marcus Frean, Lennox Leary, JP~Lewis, Kurt Wan-Duo Ma, and
  Brian McWilliams.
\newblock The shattered gradients problem: If resnets are the answer, then what
  is the question?
\newblock {\em arXiv preprint arXiv:1702.08591}, 2017.

\bibitem[BSF94]{bengio1994learning}
Yoshua Bengio, Patrice Simard, and Paolo Frasconi.
\newblock Learning long-term dependencies with gradient descent is difficult.
\newblock {\em IEEE transactions on neural networks}, 5(2):157--166, 1994.

\bibitem[CHM{\etalchar{+}}15]{choromanska2015loss}
Anna Choromanska, Mikael Henaff, Michael Mathieu, G{\'e}rard~Ben Arous, and
  Yann LeCun.
\newblock The loss surfaces of multilayer networks.
\newblock In {\em Artificial Intelligence and Statistics}, pages 192--204,
  2015.

\bibitem[HBF{\etalchar{+}}01]{hochreiter2001gradient}
Sepp Hochreiter, Yoshua Bengio, Paolo Frasconi, J{\"u}rgen Schmidhuber, et~al.
\newblock Gradient flow in recurrent nets: the difficulty of learning long-term
  dependencies, 2001.

\bibitem[Hoc91]{hochreiter1991untersuchungen}
Sepp Hochreiter.
\newblock Untersuchungen zu dynamischen neuronalen netzen.
\newblock {\em Diploma, Technische Universit{\"a}t M{\"u}nchen}, 91, 1991.

\bibitem[HR18]{hanin2018start}
Boris Hanin and David Rolnick.
\newblock How to start training: The effect of initialization and architecture.
\newblock In {\em Advances in Neural Information Processing Systems 32}, 2018.

\bibitem[HS97]{hochreiter1997long}
Sepp Hochreiter and J{\"u}rgen Schmidhuber.
\newblock Long short-term memory.
\newblock {\em Neural computation}, 9(8):1735--1780, 1997.

\bibitem[HSL16]{henaff2016recurrent}
Mikael Henaff, Arthur Szlam, and Yann LeCun.
\newblock Recurrent orthogonal networks and long-memory tasks.
\newblock In {\em Proceedings of The 33rd International Conference on Machine
  Learning}, volume~48, pages 2034--2042, 2016.

\bibitem[HZRS15]{he2015delving}
Kaiming He, Xiangyu Zhang, Shaoqing Ren, and Jian Sun.
\newblock Delving deep into rectifiers: Surpassing human-level performance on
  imagenet classification.
\newblock In {\em Proceedings of the IEEE international conference on computer
  vision}, pages 1026--1034, 2015.

\bibitem[HZRS16]{he2016deep}
Kaiming He, Xiangyu Zhang, Shaoqing Ren, and Jian Sun.
\newblock Deep residual learning for image recognition.
\newblock In {\em Proceedings of the IEEE conference on computer vision and
  pattern recognition}, pages 770--778, 2016.

\bibitem[KUMH17]{klambauer2017self}
G{\"u}nter Klambauer, Thomas Unterthiner, Andreas Mayr, and Sepp Hochreiter.
\newblock Self-normalizing neural networks.
\newblock In {\em Advances in Neural Information Processing Systems}, pages
  972--981, 2017.

\bibitem[MM15]{mishkin2015all}
Dmytro Mishkin and Jiri Matas.
\newblock All you need is a good init.
\newblock {\em arXiv preprint arXiv:1511.06422}, 2015.

\bibitem[PLR{\etalchar{+}}16]{poole2016exponential}
Ben Poole, Subhaneil Lahiri, Maithra Raghu, Jascha Sohl-Dickstein, and Surya
  Ganguli.
\newblock Exponential expressivity in deep neural networks through transient
  chaos.
\newblock In {\em Advances in neural information processing systems}, pages
  3360--3368, 2016.

\bibitem[PSG17]{pennington2017resurrecting}
Jeffrey Pennington, Samuel Schoenholz, and Surya Ganguli.
\newblock Resurrecting the sigmoid in deep learning through dynamical isometry:
  theory and practice.
\newblock In {\em Advances in neural information processing systems}, pages
  4788--4798, 2017.

\bibitem[PSG18]{pennington2018emergence}
Jeffrey Pennington, Samuel~S Schoenholz, and Surya Ganguli.
\newblock The emergence of spectral universality in deep networks.
\newblock {\em arXiv preprint arXiv:1802.09979}, 2018.

\bibitem[RPK{\etalchar{+}}17]{schoenholz2016deep}
Maithra Raghu, Ben Poole, Jon~M. Kleinberg, Surya Ganguli, and Jascha
  Sohl{-}Dickstein.
\newblock On the expressive power of deep neural networks.
\newblock In {\em Proceedings of the 34th International Conference on Machine
  Learning, {ICML} 2017}, pages 2847--2854, 2017.

\bibitem[Sch]{SchmidtHuber}
Sepp hochreiter's fundamental deep learning problem (1991).
\newblock
  \url{http://people.idsia.ch/~juergen/fundamentaldeeplearningproblem.html}.
\newblock Accessed: 2017-12-26.

\bibitem[SGS15]{srivastava2015highway}
Rupesh~Kumar Srivastava, Klaus Greff, and J{\"{u}}rgen Schmidhuber.
\newblock Highway networks.
\newblock 2015.

\bibitem[SPSD17]{schoenholz2017correspondence}
Samuel~S Schoenholz, Jeffrey Pennington, and Jascha Sohl-Dickstein.
\newblock A correspondence between random neural networks and statistical field
  theory.
\newblock {\em arXiv preprint arXiv:1710.06570}, 2017.

\bibitem[XXP17]{xie2017all}
Di~Xie, Jiang Xiong, and Shiliang Pu.
\newblock All you need is beyond a good init: Exploring better solution for
  training extremely deep convolutional neural networks with orthonormality and
  modulation.
\newblock {\em arXiv preprint arXiv:1703.01827}, 2017.

\end{thebibliography}
\newpage
\appendix

\section{Proof of Theorem \ref{T:path-moments}}\label{S:path-moments-pf}
We will use the following observation.
\begin{Lem}\label{L:key}
 Suppose $X, w_1,\ldots, w_n$ are independent real-valued random variables whose distributions are symmetric around $0.$ Assume also that the distribution of $X$ has no atoms (i.e. $\P\lr{X=x}=0$ for all $x\in \R$), and fix any bounded positive function $\psi:\R\gives \R_+$ with the property
 \begin{equation}\label{E:psi-def}
\psi(t)+\psi(-t)=1.
\end{equation}
Then for any constants $a_1,\ldots, a_n\in \R$ and any non-negative integers $k_1,\ldots, k_n$ whose sum is even, we have
\[\E{\prod_{j=1}^n w_j^{k_j} \psi(X+\sum_j w_j a_j)}=\frac{1}{2}\prod_{j=1}^n \E{w_j^{k_j}}.\]
\end{Lem}
\begin{proof}
Using that $X\stackrel{d}{=}-X,w_j\stackrel{d}{=}-w_j$ and that $\sum_j k_j$ is even, we have
  \begin{align*}
    \E{\prod_{j=1}^n w_j^{k_j} \psi(X+\sum_j w_j a_j)} & = \E{ \prod_{j=1}^n w_j^{k_j}\psi(-(X+\sum_j w_j a_j))}.
  \end{align*}
Averaging these two expressions we combine \eqref{E:psi-def} with the fact that $X$ is independent of $\set{w_j}_{j=1}^n$ and its law has no atoms to obtain the desired result.
s\end{proof}

We now turn to the proof of Theorem \ref{T:path-moments}. To this end, fix $d\geq 1,$ a collection of positive integers ${\bf n}=\lr{n_i}_{i=0}^d$, and let $\mathcal N \in \mathfrak N_{{\bf \mu},{\bf \nu}}\lr{{\bf n},d}$. Let us briefly recall the notation for paths from \S \ref{S:thms}. Given $1\leq p \leq n_0$ and $1\leq q\leq n_d,$ we defined a path $\gamma$ from p to q to be a collection $\set{\gamma(j)}_{j=0}^d$ of neurons so that $\gamma(0)=p,\, \gamma(d)=q, $ and $\gamma(j)\in \set{1,\ldots, n_j}.$ The numbers $\gamma(j)$ should be thought of as neurons in the $j^{th}$ hidden layer of $\mathcal N$. Given such a collection, we obtain for each $j$ a weight 
\begin{equation}\label{E:w-def}
w_\gamma^{(j)}:=w_{\gamma(j-1),\gamma(j)}^{(j)}
\end{equation}
between each two consecutive neurons along the path $\gamma.$ Our starting point is the expression 
\begin{equation}\label{E:gen-path-form}
Z_{p,q}=\sum_{\substack{\text{paths }\gamma \\\text{from p to q}}} \prod_{j=1}^d w_\gamma^{(j)} \,{\bf 1}_{\set{\act_{\gamma(j)}^{(j)}>0}},
\end{equation}
where $\act^{(j)}$ are defined as in \eqref{E:act-def}. This expression is well-known and follows immediately form the chain rule (c.f. e.g. equation (1) in \cite{choromanska2015loss}). We therefore have
\begin{align*}
  Z_{p,q}^{2K}=\sum_{\substack{\text{paths }\gamma_1,\ldots, \gamma_{2K} \\\text{from p to q}}} \prod_{j=1}^d \prod_{k=1}^{2K}w_{\gamma_k}^{(j)}\,{\bf 1}_{\set{\act_{\gamma_k(j)}^{(j)}>0}}.
\end{align*}
We will prove a slightly more general statement than in the formulation of Theorem \ref{T:path-moments}. Namely, suppose $\Gamma=\lr{\gamma_1,\ldots, \gamma_{2K}}$ is any collection of paths from the input of $\mathcal N$ to the output (the paths are not required to have the same starting and ending neurons) such that for every $\beta \in \Gamma(d)$, 
\[\#\setst{\gamma\in \Gamma}{\gamma(d)=\beta}\quad \text{ is even.}\] 
We will show that
\begin{equation}\label{E:gen-path-moments}
\E{\prod_{j=1}^d\prod_{k=1}^{2K}w_{\gamma_k}^{(j)}\,{\bf 1}_{\set{\act^{(j)}_{\gamma_k(j)}>0}}}=\prod_{j=1}^d \lr{\frac{1}{2}}^{\abs{\Gamma(j)}}\prod_{\substack{\alpha\in \Gamma(j-1)\\ \beta\in \Gamma(j)}} \mu_{\abs{\Gamma_{\alpha,\beta}(j)}}^{(j)}.
\end{equation}
To evaluate the expectation in \eqref{E:gen-path-moments}, note that the computation done by $\mathcal N$ is a Markov chain with respect to the layers (i.e. given $\Act^{(j-1)}$, the activations at layers $j,\ldots, d$ are independent of the weight and biases up to and including layer $j-1.$) Hence, denoting by $\mathcal F_{\leq d-1}$ the sigma algebra generated by the weight and biases up to and including layer $d-1$, the tower property for expectation and the Markov property yield
\begin{align}
\notag&\E{\prod_{j=1}^d\prod_{k=1}^{2K}w_{\gamma_k}^{(j)}\,{\bf 1}_{\set{\act_{\gamma_k(j)}^{(j)}>0}}}\\
\notag&\qquad = \E{\prod_{j=1}^{d-1}\prod_{k=1}^{2K}w_{\gamma_k}^{(j)}\,{\bf 1}_{\set{\act_{\gamma_k(j)}^{(j)}>0}}\E{\prod_{k=1}^{2K}w_{\gamma_k}^{(d)}\,{\bf 1}_{\set{\act_{\gamma_k(d)}^{(d)}>0}}~\big|~\mathcal F_{\leq d-1}}}\\
\label{E:inductive-form}&\qquad = \E{\prod_{j=1}^{d-1}\prod_{k=1}^{2K}w_{\gamma_k}^{(j)}\,{\bf 1}_{\set{\act_{\gamma_k(j)}^{(j)}>0}}\E{\prod_{k=1}^{2K}w_{\gamma_k}^{(d)}\,{\bf 1}_{\set{\act_{\gamma_k(d)}^{(d)}>0}}~\big|~\Act^{d-1}}}.
\end{align}
Next, observe that for each $1\leq j\leq d,$ conditioned on $\Act^{(j-1)}$, the families of random variables $\set{w_{\alpha, \beta}^{(j)},\, \act_\beta^{(j)}}_{\alpha=1}^{n_{j-1}}$ are independent for different $\beta.$ For $j=d$ this implies
\begin{equation}\label{E:j-prod}
\E{\prod_{k=1}^{2K}w_{\gamma_k}^{(d)} \,{\bf 1}_{\set{\act_{\gamma_k(d)}^{(d)}>0}}~\big|~\Act^{(d-1)}}=\prod_{\beta \in \Gamma(d)} \E{\prod_{\substack{k=1\\ \gamma_k(d)=\beta}}^{2K} w_{\gamma_k}^{(d)}\,{\bf 1}_{\set{\act_{\gamma_k(d)}^{(d)}>0}}~\bigg|~\Act^{(d-1)}}.
\end{equation}
Consider the decomposition
\begin{equation}\label{E:decomp}
\act_\beta^{(d)}=\act_{\Gamma,\beta}^{(d)}~~+~~\widehat{\act}_{\Gamma, \beta}^{(d)},
\end{equation}
where
\begin{align*}
\act_{\Gamma, \beta}^{(d)}&:=\sum_{\alpha \in \Gamma(d-1)} \Act_{\alpha}^{(d-1)}w_{\alpha, \beta}^{(d)}\\
\widehat{\act}_{\Gamma, \beta}^{(d)}&=\act_\beta^{(d)}-\act_{\Gamma,\beta}^{(d)}=b_\beta^{(d)}+\sum_{\alpha \not\in \Gamma(d-1)} \Act_{\alpha}^{(d-1)}w_{\alpha, \beta}^{(d)}.
\end{align*}
Let us make several observations about $\widehat{\act}_{\Gamma,\beta}^{(d)}$ and $\act_{\Gamma, \beta}^{(d)}$  \textit{ when conditioned on $\Act^{(d-1)}$}. First, the conditioned random variable $\widehat{\act}_{\Gamma,\beta}^{(d-1)}$ is independent of the conditioned random variable $\act_{\Gamma,\beta}^{(d-1)}$. Second, the distribution of $\widehat{\act}_{\Gamma, \beta}^{(d)}$ conditioned on $\Act^{(d-1)}$ is symmetric around $0.$ Third, since we assumed that the bias distributions $\nu^{(j)}$ for $\mathcal N$ have no atoms, the conditional distribution of $\widehat{\act}_{\Gamma,\beta}^{(d)}$ also has no atoms. Fourth, $\act_{\Gamma,\beta}^{(d-1)}$ is a linear combination of the weights $\set{w_{\alpha, \beta}^{(d)}}_{\alpha \in \Gamma(j-1)}$ with given coefficients $\set{\Act_{\alpha}^{(d-1)}}_{\alpha \in \Gamma(j-1)}.$ Since the weight distributions $\mu^{(j)}$ for $\mathcal N$ are symmetric around $0,$ the above five observations, together with \eqref{E:decomp} allow us to apply Lemma \ref{L:key} and to conclude that
\begin{align}\label{E:inductive-key}
&\E{\prod_{k=1}^{2K}w_{\gamma_k}^{(j)}\,{\bf 1}_{\set{\act_{\gamma_k(d)}^{(d)}>0}} ~\big|~\Act^{(d-1)}}=\lr{\frac{1}{2}}^{\abs{\Gamma(d)}}\prod_{\substack{\beta\in \Gamma(d)\\\alpha\in \Gamma(d-1)}}\mu_{\abs{\Gamma_{\alpha,\beta}(d)}}^{(d)}.
\end{align}
Combining this with \eqref{E:inductive-form} yields 
\begin{align*}
  &\E{\prod_{j=1}^d\prod_{k=1}^{2K}w_{\gamma_k}^{(j)}\,{\bf 1}_{\set{\act_{\gamma_k(j)}^{(j)}>0}}}\\
&\qquad=\E{\prod_{j=1}^{d-1}\prod_{k=1}^{2K}w_{\gamma_k}^{(j)}\,{\bf 1}_{\set{\act_{\gamma_k(j)}^{(j)}>0}}}\lr{\frac{1}{2}}^{\abs{\Gamma(d)}}\prod_{\substack{\beta\in \Gamma(d)\\\alpha\in \Gamma(d-1)}}\mu_{\abs{\Gamma_{\alpha,\beta}(d)}}^{(d)}.
\end{align*}
To complete the argument, we must consider two cases. First, recall that by assumption, for every $\beta\in \Gamma(d),$ the number of $\gamma\in \Gamma$ for which $\gamma(d)=\beta$ is even. If for every $j\leq d$ and each $\alpha\in \Gamma(j-1)$ the number of $\gamma\in \Gamma$ passing through $\alpha$ is even, then we may repeat the preceding argument to directly obtain \eqref{E:gen-path-moments}. Otherwise, we apply this argument until we reach $\alpha\in \Gamma(j-1),\,\beta\in \Gamma(j)$ so that the number $\abs{\Gamma_{\alpha, \beta}(j)}$ of paths in $\Gamma$ that pass through $\alpha$ and $\beta$ is odd. In this case, the right hand side of \eqref{E:inductive-form} vanishes since the measure $\mu^{(d)}$ is symmetric around $0$ and thus has vanishing odd moments. Relation \eqref{E:gen-path-moments} therefore again holds since in this case both sides are $0.$ This completes the proof of Theorem \ref{T:path-moments}. \hfill $\square$

\section{Proof of Theorem \ref{T:slope}}\label{S:slope-pf} In this section, we use Theorem \ref{T:path-moments} to prove Theorem \ref{T:slope}. Let us first check \eqref{E:Z2}. According to Theorem \ref{T:path-moments}, we have
\[\E{Z_{p,q}^2}=\sum_{\substack{\Gamma=\lr{\gamma_1,\gamma_2}\\ \text{paths from p to q}}}\prod_{j=1}^d\lr{\frac{1}{2}}^{\abs{\Gamma(j)}} \prod_{\substack{\alpha\in \Gamma(j-1)\\\beta \in \Gamma(j)}}\mu_{\abs{\Gamma_{\alpha,\beta}(d)}}^{(j)}.\]
Note that since $\mu$ is symmetric around $0,$ we have that $\mu_1=0.$ Thus, the terms where $\gamma_1\neq \gamma_2$ vanish. Using $\mu_2^{(j)}=\frac{2}{n_{j-1}}$, we find
\begin{align*}
\E{Z_{p,q}^2}&=\sum_{\substack{\text{paths }\gamma\\ \text{from p to q}}}\prod_{j=1}^d\frac{1}{2} \cdot \frac{2}{n_{j-1}}=\frac{1}{n_0},
\end{align*}
as claimed. We now turn to proving \eqref{E:Z4}. Using Theorem \ref{T:path-moments}, we have
\begin{align*}
\E{Z_{p,q}^4}&=\sum_{\substack{\Gamma=\lr{\gamma_k}_{k=1}^4\\ \text{paths from p to q}}}\prod_{j=1}^d\lr{\frac{1}{2}}^{\abs{\Gamma(j)}} \prod_{\substack{\beta\in \Gamma(j)\\\alpha \in \Gamma(j-1)}}\mu_{\abs{\Gamma_{\alpha,\beta}(j)}}^{(j)}\\
&=\sum_{\substack{\Gamma=\lr{\gamma_k}_{k=1}^4\\ \text{paths from p to q}\\\abs{\Gamma_{\alpha,\beta}(j)}\text{ even }\forall \alpha,\beta}} \prod_{j=1}^d \lr{\frac{\mu_4^{(j)}}{2}{\bf 1}_{\left\{\substack{\abs{\Gamma(j-1)}=1\\\abs{\Gamma(j)}=1}\right\}}+\frac{\lr{\mu_2^{(j)}}^2}{2}{\bf 1}_{\left\{\substack{\abs{\Gamma(j-1)}=2\\ \abs{\Gamma(j)}=1}\right\}}+\frac{\lr{\mu_2^{(j)}}^2}{4}{\bf 1}_{\left\{\abs{\Gamma(j)}=2\right\}}} ,
\end{align*}
where we have used that $\mu_1^{(j)}=\mu_3^{(j)}=0$. Fix $\bar{\Gamma}=\lr{\gamma_k}_{k=1}^4$. Note that $\bar{\Gamma}$ gives a non-zero contribution to $\E{Z_{p,q}^4}$ only if 
\[\abs{\bar{\Gamma}_{\alpha,\beta}(j)}\text{ is even},\qquad \forall j,\alpha, \beta.\]
For each such $\bar{\Gamma}$, we have $\abs{\bar{\Gamma(j)}}\in\set{1,2}$ for every $j$. Hence, for every $\bar{\Gamma}$ that contributes a non-zero term in the expression above for $\E{Z_{p,q}^4}$, we may find a collection of two paths $\Gamma=\lr{\gamma_1,\gamma_2}$ from p to q such that 
\[\bar{\Gamma}(j)=\Gamma(j),\qquad \abs{\bar{\Gamma}_{\alpha,\beta}(j)}=2\abs{\Gamma_{\alpha,\beta}(j)},\quad \forall j,\alpha, \beta.\]
We can thus write $\E{Z_{p,q}^4}$ as
\begin{align}\label{E:Z4-2path}
\sum_{\substack{\Gamma=\lr{\gamma_1,\gamma_2}\\ \text{paths from p to q}}} A(\Gamma)\prod_{j=1}^d \lr{\frac{\mu_4^{(j)}}{2}{\bf 1}_{\left\{\substack{\abs{\Gamma(j-1)}=1\\\abs{\Gamma(j)}=1}\right\}}+\frac{\lr{\mu_2^{(j)}}^2}{2}{\bf 1}_{\left\{\substack{\abs{\Gamma(j-1)}=2\\ \abs{\Gamma(j)}=1}\right\}}+\frac{\lr{\mu_2^{(j)}}^2}{4}{\bf 1}_{\left\{\abs{\Gamma(j)}=2\right\}}},
\end{align}
where we introduced
\begin{equation}\label{E:A-def}
A(\Gamma):=\frac{\# \left\{\bar{\Gamma}=\lr{\bar{\gamma}_k}_{k=1}^4,\,\, \bar{\gamma}_k\text{ path from p to q}~\big|~\substack{\forall j,\alpha, \beta,\,\,\,\Gamma(j)=\bar{\Gamma}(j)\\ 2\abs{\Gamma_{\alpha, \beta}(j)}=\abs{\bar{\Gamma}_{\alpha, \beta}(j)}\}}\right\}}{\# \left\{\bar{\Gamma}=\lr{\bar{\gamma}_k}_{k=1}^2,\,\, \bar{\gamma}_k\text{ path from p to q}~\big|~\substack{\forall j,\alpha, \beta,\,\,\,\Gamma(j)=\bar{\Gamma}(j)\\ \abs{\Gamma_{\alpha, \beta}(j)}=\abs{\bar{\Gamma}_{\alpha, \beta}(j)}\}}\right\}},
\end{equation}
which we now evaluate.
\begin{lemma}\label{L:A-formula}
  For each $\Gamma=\lr{\gamma_k}_{k=1}^2$ with $\gamma_k$ paths from p to q, we have
  \begin{equation}\label{E:A-formula}
A(\Gamma)= 3^{\#\setst{j}{\abs{\Gamma(j-1)}=1,\,\abs{\Gamma(j)}=2}}=3^{\#\setst{j}{\abs{\Gamma(j-1)}=2,\,\abs{\Gamma(j)}=1}}.
\end{equation}
\end{lemma}
\begin{proof}
We begin by checking the first equality in \eqref{E:A-formula} by induction on $d.$ Fix $\Gamma=\lr{\gamma_1,\gamma_2}.$ When $d=1,$ we have $\abs{\Gamma(0)}=\abs{\Gamma(1)}=1$. Hence $\gamma_1=\gamma_2$ and $A(\Gamma)=1$ since both the numerator and denominator on the right hand side of \eqref{E:A-def} equal $1$. The right hand side of \eqref{E:A-formula} is also $1$ since $\abs{\Gamma(j)}=1$ for every $j.$ This completes the base case. Suppose now that $D\geq 2,$ and we have proved \eqref{E:A-formula} for all $d\leq D-1.$ Let 
\[j_*:=\min\setst{j=1,\ldots, d}{\abs{\Gamma(j)}=1}.\]
If $j_*=1$, then we are done by the inductive hypothesis. Otherwise, there are two choices of
$\bar{\Gamma}=\set{\bar{\gamma}_k}_{k=1}^2$ for which
\[\Gamma(j)=\bar{\Gamma}(j),\qquad \abs{\Gamma_{\alpha, \beta}(j)}=\abs{\bar{\Gamma}_{\alpha, \beta}(j)},\qquad j\leq j_*.\]
These choices correspond to the two permutations of $\set{\gamma_k}_{k=1}^2.$ Similarly, there are $6$ choices of $\bar{\Gamma}=\set{\bar{\gamma}_k}_{k=1}^4$ for which  
\[\Gamma(j)=\bar{\Gamma}(j),\qquad 2\abs{\Gamma_{\alpha, \beta}(j)}=\abs{\bar{\Gamma}_{\alpha, \beta}(j)},\qquad j\leq j_*.\]
The six choices correspond to selecting one of two choices for $\gamma_1(1)$ and three choices of an index $k=2,3,4$ so that $\gamma_k(j)$ coincides with $\gamma_1(j)$ for each $j\leq j_*.$ If $j_*=d$, we are done. Otherwise, we apply the inductive hypothesis to paths from $\Gamma(j_*)$ to $\Gamma(d)$ to complete the proof of the first equality in \eqref{E:A-formula}. The second equality in \eqref{E:A-formula} follows from the observation that since $\abs{\Gamma(0)}=\abs{\Gamma(d)}=1,$ the number of $j\in \set{1,\ldots, d}$ for which $\abs{\Gamma(j-1)}=1,\, \abs{\Gamma(j)}=2$ must equal the number of $j$ for which $\abs{\Gamma(j-1)}=2,\, \abs{\Gamma(j)}=1$.
\end{proof}
\noindent Combining \eqref{E:Z4-2path} with  \eqref{E:A-formula}, we may write $\E{Z_{p,q}^4}$ as
\begin{equation}\label{E:4path-reduced}
\sum_{\substack{\Gamma=\lr{\gamma_1,\gamma_2}\\ \text{paths from p to q}}} \prod_{j=1}^d \bigg[\frac{\mu_4^{(j)}}{2}{\bf 1}_{\left\{\substack{\abs{\Gamma(j-1)}=1\\\abs{\Gamma(j)}=1}\right\}}+\frac{3}{2}\lr{\mu_2^{(j)}}^2{\bf 1}_{\left\{\substack{\abs{\Gamma(j-1)}=2\\ \abs{\Gamma(j)}=1}\right\}}+\frac{\lr{\mu_2^{(j)}}^2}{4}{\bf 1}_{\left\{\abs{\Gamma(j)}=2\right\}}\bigg].
\end{equation}
Observe that since $\mu_2^{(j)}=2/n_{j-1}$, we have
\[\lr{\#\set{\Gamma=\lr{\gamma_k}_{k=1}^2\text{ paths from p to q}}}^{-1}=\prod_{j=1}^{d-1} \frac{1}{n_j^2}= n_0^{2}\cdot \prod_{j=1}^d \lr{\mu_2^{(j)}}^2/4.\]
Hence, 
$$ \E{Z_{p,q}^4}=\frac{1}{n_0^2}\E{X_d\lr{\gamma_1,\gamma_2}},$$
where the expectation on the right hand side is over the uniform measure on paths $(\gamma_1,\gamma_2)$ from the input of $\mathcal N$ to the output conditioned on $\gamma_1(0)=\gamma_2(0)=p$ and $\gamma_1(d)=\gamma_2(d)=q,$ and 
\[X_d(\gamma_1,\gamma_2):=\prod_{j=1}^d\lr{2\twiddle{\mu}_4\cdot {\bf 1}_{\left\{\substack{\abs{\Gamma(j-1)}=1\\\abs{\Gamma(j)}=1}\right\}} + 6\cdot {\bf 1}_{\left\{\substack{\abs{\Gamma(j-1)}=2\\ \abs{\Gamma(j)}=1}\right\}} + {\bf 1}_{\left\{\abs{\Gamma(j)}=2\right\}}},~~ \Gamma = \lr{\gamma_1,\gamma_2}.\]
We now obtain the upper and lower bounds in \eqref{E:Z4} on $\E{Z_{p,q}^4}$ in similar ways. In both cases, we use the observation that the number of $\Gamma=\lr{\gamma_1,\gamma_2}$ for which $\abs{\Gamma(j)}=1$ for exactly $k$ values of $1\leq j\leq d-1$ is
\[\prod_{j=1}^{d-1} n_j \cdot \sum_{\substack{I\subseteq \set{1,\ldots, d-1}\\ \abs{I}=d-1-k}}\prod_{j\in I}(n_j-1).\]
The value of $X_d$ corresponding to every such path is at least $2^{k+1}$ since $\twiddle{\mu}_4^{(j)}\geq 1$ for every $j$ and is at most $6\twiddle{\mu}_{4,max}$ for the same reason. Therefore, using that for all $\ep\in [0,1],$ we have 
\[\log(1+\ep)\geq \frac{\ep}{2},\]
we obtain
\begin{align*}
\E{X_d}&~~\geq ~~\frac{1}{\prod_{j=1}^{d-1} n_j}~~ \sum_{k=0}^{d-1} ~2^{k+1} \sum_{\substack{I\subseteq \set{1,\ldots, d-1}\\ \abs{I}=d-1-k}}\prod_{j\in I}(n_j-1)\\
&= 2\sum_{k=0}^{d-1} ~~\sum_{\substack{I\subseteq \set{1,\ldots d-1}\\ \abs{I}=d-1-k}}~ \prod_{j\not \in I}\lr{\frac{2}{n_j}}\prod_{j\in I} \lr{1+\frac{1}{n_j}}\\
&  = 2\prod_{j=1}^{d-1}\lr{1 + \frac{1}{n_j}}~~\geq~~ 2\exp\lr{\frac{1}{2}\sum_{j=1}^{d-1} \frac{1}{n_j}}.
\end{align*}
This completes the proof of the lower bound. The upper bound is obtained in the same way:
\begin{align}
\notag \E{X_d}&~~\leq ~~\frac{1}{\prod_{j=1}^{d-1} n_j}~~ \sum_{k=0}^{d-1} ~\lr{6\twiddle{\mu}_{4,max}}^{k+1} \sum_{\substack{I\subseteq \set{1,\ldots, d-1}\\ \abs{I}=d-1-k}}\prod_{j\in I}(n_j-1)\\
\label{E:upper-bound}&  = 6\twiddle{\mu}_{4,max}\prod_{j=1}^{d-1}\lr{1 + \frac{6\twiddle{\mu}_{4,max}}{n_j}}~~\leq~~ 6\twiddle{\mu}_{4,max}\exp\lr{6\twiddle{\mu}_{4,max}\sum_{j=1}^{d-1} \frac{1}{n_j}}.
\end{align}
The upper bounds for $\E{Z_{p,q}^{2K}}$ for $K\geq 3$ are obtained in essentially the same way. Namely, we return to the expression for $\E{Z_{p,q}^{2K}}$ provided by Theorem \ref{T:path-moments}:
\[\E{Z_{p,q}^{2K}}=\sum_{\substack{\Gamma = \set{\gamma_k}_{k=1}^{2K}\\ \gamma_k \text{ paths from p to q}}} \prod_{j=1}^d \lr{\frac{1}{2}}^{\abs{\Gamma(j)}}\prod_{\substack{\beta\in \Gamma(j)\\ \alpha \in \Gamma(j-1)}} \mu_{\abs{\Gamma_{\alpha, \beta}(j)}}^{(j)}.\]
As with the second and fourth moment computations, we note that $\mu_{\abs{\Gamma_{\alpha, \beta}(j)}}$ vanishes unless each $\abs{\Gamma_{\alpha, \beta}(j)}$ is even. Hence, as with \eqref{E:4path-reduced}, we may write
\begin{equation}\label{E:Kpath-reduced}
\E{Z_{p,q}^{2K}}=\sum_{\substack{\Gamma = \set{\gamma_k}_{k=1}^{K}\\ \gamma_k \text{ paths from p to q}}} A_K(\Gamma)\prod_{j=1}^d \lr{\frac{1}{2}}^{\abs{\Gamma(j)}}\prod_{\substack{\beta\in \Gamma(j)\\ \alpha \in \Gamma(j-1)}} \mu_{2\abs{\Gamma_{\alpha, \beta}(j)}}^{(j)},
\end{equation}
where $A_K(\Gamma)$ is the analog of $A(\Gamma)$ from \eqref{E:A-def}. The same argument as in Lemma \ref{L:A-formula} shows that 
\[A_K(\Gamma)\leq \lr{\frac{(2K)!}{K!}}^{\#\setst{1\leq j\leq d}{\abs{\Gamma(j)}<K}}.\]
Combining this with
\[\lr{\#\set{\Gamma=\lr{\gamma_k}_{k=1}^K\text{ paths from p to q}}}^{-1}=\prod_{j=1}^{d-1} \frac{1}{n_j^K}= n_0^{K}\cdot \prod_{j=1}^d \lr{\mu_2^{(K)}}^2/2^K,\]
which is precisely the weight in \eqref{E:Kpath-reduced} assigned to collections $\Gamma$ with $\abs{\Gamma(j)}=K$ for every $1\leq j \leq d-1,$ yields
\[\E{Z_{p,q}^{2K}}\leq \frac{1}{n_0^K}\E{X_d\lr{\gamma_1,\ldots, \gamma_K}~|~ \gamma_k(0)=p,\,\, \gamma_k(d)=q},\]
where the expectation is over uniformly chosen collections $\Gamma=\lr{\gamma_1,\ldots, \gamma_K}$ of paths from the input to the output of $\mathcal N$ and 
\[X_d(\Gamma)=\prod_{j=1}^d 2^{K-\abs{\Gamma(j)}} \frac{(2K)!}{K!}\prod_{\substack{\alpha \in \Gamma(j-1)\\\beta\in \Gamma(j)\\\abs{\Gamma(j)}<K}} \mu_{2\abs{\Gamma_{\alpha, \beta}(j)}}^{(j)}.\]
To complete the proof of the upper bound for $\E{Z_{p,q}^{2K}}$ we now proceed just as the upper bound for the $4^{th}$ moment computation. That is, given $K< \min\set{n_j}$, the number of collections of paths $\Gamma=\lr{\gamma_k}_{k=1}^{K}$ which $\abs{\Gamma(j)}<K$ for exactly $m$ values of $j$ is bounded above by
\[\prod_{j=1}^{d-1}n_j^{K-1} \sum_{\substack{I\subseteq \set{1,\ldots, d-1}\\ \abs{I}=d-1-m}}\prod_{j\in I}\lr{n_j-K}.\]
The value of $X_d$ on each such collection is at most $\lr{C_K}^m$, where $C_K=2^{K-1} \frac{(2K)!}{K!}$ is a large but fixed constant. Hence, just as in \eqref{E:upper-bound},
\[\E{X_d(\Gamma)} \leq C_K \exp\lr{C_K \sum_{j=1}^{d-1}\frac{1}{n_j}}.\]
This completes the proof of Theorem \ref{T:slope}. \hfill $\square$

\section{Proof of Theorem \ref{T:quenched}}\label{S:quenched-pf}
\noindent We have
\begin{equation}\label{E:var-diag}
\widehat{\Var[Z^2]}=\frac{1}{M}\lr{1-\frac{1}{M}}\sum_{m=1}^MZ_{p_m,q_m}^4 - \frac{1}{M^2}\sum_{m_1\neq m_2} Z_{p_{m_1},q_{m_1}}^2 Z_{p_{m_2},q_{m_2}}^2. 
\end{equation}
Fixing $p,q$ and using that the second sum in the previous line has $M(M-1)$ terms, we have
\begin{align*}
-  \frac{1}{M^2}\sum_{m_1\neq m_2} Z_{p_{m_1},q_{m_1}}^2 Z_{p_{m_2},q_{m_2}}^2 & =  \frac{1}{M^2} \sum_{m_1\neq m_2}\lr{Z_{p,q}^4- Z_{p_{m_1},q_{m_1}}^2 Z_{p_{m_2},q_{m_2}}^2} + \lr{1-\frac{1}{M}} Z_{p,q}^4. 
\end{align*}
Hence, using that $\E{Z_{p,q}^4}$ is independent of the particular values of $p,q$, we fix some $p,q$ and write 
\begin{align}
\label{E:reduced-empvar}\E{\widehat{\Var[Z^2]}}&=\frac{1}{M^2}\sum_{m_1\neq m_2} \E{Z_{p,q}^4} -\E{ Z_{p_{m_1},q_{m_1}}^2 Z_{p_{m_2},q_{m_2}}^2}. 
\end{align}
To estimate the difference in this sum, we use Theorem \ref{T:path-moments} to obtain
\begin{align}
\label{E:4th-power}  \E{Z_{p,q}^4} & = \sum_{\substack{\Gamma=\lr{\gamma_k}_{k=1}^4\\ \gamma_k:p\gives q}} \prod_{j=1}^d \lr{\frac{1}{2}}^{\abs{\Gamma(j)}}\prod_{\alpha,\beta} \mu_{\abs{\Gamma_{\alpha, \beta}(j)}}^{(j)}= \sum_{\substack{\Gamma=\lr{\gamma_k}_{k=1}^4\\ \gamma_k:p\gives q}} \prod_{j=1}^d C_j(\Gamma)\\
\label{E:2nd-power2}\E{Z_{p_1,q_1}^2 Z_{p_2,q_2}^2}&=\sum_{\substack{\bar{\Gamma}=\lr{\gamma_k}_{k=1}^4\\ \gamma_1,\gamma_2:p_1\gives q_1\\ \gamma_3,\gamma_4:p_2\gives q_2}} \prod_{j=1}^d \lr{\frac{1}{2}}^{\abs{\bar{\Gamma}(j)}}\prod_{\alpha,\beta} \mu_{\abs{\bar{\Gamma}_{\alpha, \beta}(j)}}^{(j)}=\sum_{\substack{\bar{\Gamma}=\lr{\gamma_k}_{k=1}^4\\ \gamma_1,\gamma_2:p_1\gives q_1\\ \gamma_3,\gamma_4:p_2\gives q_2}} \prod_{j=1}^d C_j(\bar{\Gamma}).
\end{align}
Note that since the measures $\mu^{(j)}$ of the weights are symmetric around zero, their odd moments vanish and hence the only non-zero terms in \eqref{E:4th-power} and \eqref{E:2nd-power2} are those for which 
\[\abs{\Gamma(j)},\, \abs{\bar{\Gamma}(j)}\in \set{1,2},\quad \abs{\Gamma_{\alpha, \beta}(j)},\, \abs{\bar{\Gamma}_{\alpha,\beta}(j)}\in \set{2,4},\qquad \forall j,\alpha, \beta.\]
Further, observe that each path $\gamma$ from some fixed input neuron to some fixed output vertex is determined uniquely by the sequence of hidden neurons $\gamma(j)\in \set{1,\ldots, n_j}$ through which it passes for $j=1,\ldots, d-1$. Therefore, we may identify each collection of paths $\Gamma=\lr{\gamma_k}_{k=1}^4$ in the sum \eqref{E:4th-power} with a unique collection of paths $\bar{\Gamma}=\lr{\bar{\gamma}_k}_{k=1}^4$ in \eqref{E:2nd-power2} by asking that $\gamma_k(j)=\bar{\gamma}_k(j)$ for each $k$ and all $1\leq j \leq d-1$. Observe further that under this bijection, 
\begin{equation}
j\neq 1,d \quad \Rightarrow \quad C_j(\Gamma)=C_j(\bar{\Gamma}).\label{E:Cjequal}
\end{equation}
For $j=1,d$, the terms $C_j(\Gamma)$ and $C_j(\bar{\Gamma})$ are related as follows:
\begin{align}
\label{E:C1} C_1(\Gamma) &= C_1(\bar{\Gamma})\lr{{\bf 1}_{\set{\abs{\Gamma(1)}=2}} + \twiddle{\mu}_4^{(1)} \cdot {\bf 1}_{\set{\abs{\Gamma(1)}=1}}}\\
\label{E:Cd} C_d(\Gamma) &= C_d(\bar{\Gamma})\lr{{\bf 1}_{\set{\abs{\bar{\Gamma}(d)}=1}}+ 2\cdot {\bf 1}_{\left\{\substack{\abs{\bar{\Gamma}(d)}=2\\ \abs{\Gamma(d-1)}=2} \right\} }  + 2\twiddle{\mu}_4^{(d)} \cdot {\bf 1}_{\left\{\substack{\abs{\bar{\Gamma}(d)}=2\\ \abs{\Gamma(d-1)}=1} \right\} } } .
\end{align}
We consider two cases: (i) $q_{m_1}\neq q_{m_2}$ (i.e. $\abs{\bar{\Gamma}(d)}=2$) and (ii) $q_{m_1}=q_{m_2}$ (i.e. $\abs{\bar{\Gamma}(d)}=1$ and $p_{m_1}\neq p_{m_2}$). In case (i), we have
\[
 C_1(\Gamma) = C_1(\bar{\Gamma})\lr{{\bf 1}_{\set{\abs{\Gamma(1)}=2}} + \twiddle{\mu}_4^{(1)} {\bf 1}_{\set{\abs{\Gamma(1)}=1}}}\geq C_1(\bar{\Gamma})\qquad \text{and}\qquad C_d(\Gamma) \geq 2 C_d(\bar{\Gamma}).\]
Hence, using \eqref{E:C1} and \eqref{E:Cd}, we find that in case (i)
\[q_{m_1}\neq q_{m_2}\qquad \Rightarrow\qquad \E{Z_{p_m,q_m}^4} \geq 2 \E{Z_{p_{m_1},q_{m_1}}^2 Z_{p_{m_2},q_{m_2}}^2}.\]
In case (i) we therefore find
\begin{equation}
  \label{E:case1-diff-est}
  \E{Z_{p_m,q_m}^4} -\E{Z_{p_{m_1},q_{m_1}}^2} \geq \E{Z_{p_{m_1},q_{m_1}}^2} \geq \frac{1}{n_0^2}\exp\lr{\frac{1}{2}\sum_{j=1}^{d-1}\frac{1}{n_j}},
\end{equation}
where the last estimate is proved by the same argument as the relation \eqref{E:Z4} in Theorem \ref{T:slope}.  To obtain the analogous lower bound for case (ii), we write $q=q_{m_1}= q_{m_2},\, p_{m_1}\neq p_{m_2}.$ In this case, combining \eqref{E:Cjequal} with \eqref{E:Cd}, we have
\[C_j(\Gamma)=C_j(\bar{\Gamma})\quad j=2,\ldots, d.\]
Moreover, continuing to use the bijection between $\Gamma$ and $\bar{\Gamma}$ above, \eqref{E:C1} yields in this case
\[C_1(\bar{\Gamma})=
\begin{cases}
\frac{1}{\twiddle{\mu}_4^{(1)}} C_1(\Gamma) &\quad, \qquad\text{ if } \abs{\Gamma(1)}=1\\
0&\quad, \qquad\text{ if }\abs{\Gamma(1)}=2 
\end{cases}.
\]
Hence, $\E{Z_{p,q}^4}-\E{Z_{p_1,q}^2Z_{p_2,q}^2}$ becomes
\begin{align*}
 \sum_{\substack{\Gamma=\lr{\gamma_k}_{k=1}^4\\ \gamma_k:p\gives q}} \lr{C_1(\Gamma)-C_1(\bar{\Gamma})}\prod_{j=2}^d C_j(\Gamma)~~=~~ \lr{1-\frac{1}{\twiddle{\mu}_4^{(1)}}} \sum_{\substack{\Gamma=\lr{\gamma_k}_{k=1}^4\\ \gamma_k:p\gives q\\\abs{\Gamma(1)}=1}} \prod_{j=1}^d C_j(\Gamma).
\end{align*}
Using that if $\abs{\Gamma(0)}=\abs{\Gamma(1)}=1,$ then 
\[C_1(\Gamma)= \frac{\mu_4^{(1)}}{2}= \frac{2 \twiddle{\mu}_4^{(1)}}{n_0^2},\]
we find
\begin{equation}
\E{Z_{p,q}^4}-\E{Z_{p_1,q}^2Z_{p_2,q}^2}= \frac{2}{n_0^2}\lr{\twiddle{\mu}_4^{(1)}-1} \sum_{\substack{\Gamma=\lr{\gamma_k}_{k=1}^4\\ \gamma_k:p\gives q\\\abs{\Gamma(1)}=1}} \prod_{j=2}^d C_j(\Gamma).\label{E:case2-diff}
\end{equation}
Writing $\widehat{p}$ for any neuron in the first hidden layer of $\mathcal N$, we rewrite the sum in the previous line as
\[\sum_{\substack{\Gamma=\lr{\gamma_k}_{k=1}^4\\ \gamma_k:p\gives q\\\abs{\Gamma(1)}=1}} \prod_{j=2}^d C_j(\Gamma)= n_1 \sum_{\substack{\Gamma=\lr{\gamma_k}_{k=1}^4\\ \gamma_k:\widehat{p}\gives q}} \prod_{j=2}^d C_j(\Gamma)= n_1 \E{Z_{\widehat{p},q}^4},\]
where the point is now that we are considering paths only from $\widehat{p}$ to $q$.  According to \eqref{E:Z4} from Theorem \ref{T:slope}, we have
\[\E{Z_{\widehat{p},q}^4}\geq \frac{2}{n_1^2}\exp\lr{\frac{1}{2} \sum_{j=2}^{d-1}\frac{1}{n_j}}.\]
Combining this with \eqref{E:case2-diff} yields
\[\E{Z_{p,q}^4}-\E{Z_{p_1,q}^2Z_{p_2,q}^2}\geq \frac{4}{n_0^2n_1}\lr{\twiddle{\mu}_4^{(1)}-1}\exp\lr{\frac{1}{2}\sum_{j=2}^{d-1}\frac{1}{n_j}}.\]
Combining this with \eqref{E:reduced-empvar}, \eqref{E:case1-diff-est} and setting
\[\eta:=\frac{\#\setst{m_1\neq m_2}{q_{m_1}=q_{m_2}}}{M(M-1)}=\frac{(n_0-1)n_0n_d}{n_0n_d(n_0n_d-1)}=\frac{n_0-1}{n_0n_d-1},\]
we obtain
\begin{align*}
  \E{\widehat{\Var}[Z^2]} & \geq \frac{1}{n_0^2}\lr{1-\frac{1}{M}}\lr{ \eta  + \frac{4\lr{1-\eta}}{n_1}\lr{\twiddle{\mu}_4^{(1)}-1}e^{-\frac{1}{n_1}}  } \exp\lr{\frac{1}{2}\sum_{j=1}^{d-1}\frac{1}{n_j}},
\end{align*}
proving \eqref{E:quenched-est-lb}. Finally, the upper bound in \eqref{E:quenched-est-ub} follows from dropping the negative term in \eqref{E:reduced-empvar} and applying the upper bound from \eqref{E:Z4}.\hfill $\square$


\end{document}